%% file: arxiv_final.tex
\documentclass[twoside]{article}

\usepackage[round]{natbib}

\usepackage{pifont}
\newcommand{\cmark}{\ding{51}}%
\newcommand{\xmark}{\ding{55}}%

\usepackage{fullpage}
\usepackage[utf8]{inputenc} 
\usepackage[T1]{fontenc}    
\usepackage{hyperref}
\usepackage{booktabs}       
\usepackage{amsfonts,amsmath,amsthm,amssymb}       
\usepackage{nicefrac}       
\usepackage{microtype}      
\usepackage{graphicx}

\usepackage{hyperref}       
\usepackage{url}            
\usepackage{booktabs}       
\usepackage{amsfonts,amsmath,amsthm}       
\usepackage{nicefrac}       
\usepackage{microtype}      
\usepackage{graphicx}
\usepackage{todonotes}
\usepackage{bm}
\usepackage{soul}
\usepackage{mathrsfs}
\usepackage[capposition=bottom]{floatrow}

\input{math_macros}

\title{Harmonizable mixture kernels with variational Fourier features}

\author{ Zheyang Shen \qquad Markus Heinonen \qquad Samuel Kaski  \\ Helsinki Institute for Information Technology HIIT \\ Aalto University}

\begin{document}
\twocolumn[\maketitle]

\begin{abstract}
The expressive power of Gaussian processes depends heavily on the choice of kernel. In this work we propose the novel harmonizable mixture kernel (HMK), a family of expressive, interpretable, non-stationary kernels derived from mixture models on the generalized spectral representation. As a theoretically sound treatment of non-stationary kernels, HMK supports harmonizable covariances, a wide subset of kernels including all stationary and many non-stationary covariances. We also propose variational Fourier features, an inter-domain sparse GP inference framework that offers a representative set of `inducing frequencies'. We show that harmonizable mixture kernels interpolate between local patterns, and that variational Fourier features offers a robust kernel learning framework for the new kernel family. 
\end{abstract}

\section{INTRODUCTION}

Kernel methods are one of the cornerstones of machine learning and pattern recognition. Kernels, as a measure of similarity between two objects, depart from common linear hypotheses by allowing for complex nonlinear patterns \citep{vapnik2013nature}. In a Bayesian framework, kernels are interpreted probabilistically as covariance functions of random processes, such as for the Gaussian processes (GP) in Bayesian nonparametrics.
As rich distributions over functions, GPs serve as an intuitive nonparametric inference paradigm, with well-defined posterior distributions. \par
The kernel of a GP encodes the prior knowledge of the underlying function. 
The \emph{squared exponential} (SE) kernel is a common choice which, however, can only model global monotonic covariance patterns, while generalisations have explored local monotonicities \citep{gibbs1998bayesian, paciorek2004nonstationary}.
In contrast, expressive kernels can learn hidden representations of the data \citep{wilson2013gaussian}.\par
The two main approaches to construct expressive kernels are composition of simple kernel functions \citep{archambeau2011multiple, durrande2016detecting, gonen2011multiple, rasmussen2006, sun2018differentiable}, and modelling of the spectral representation of the kernel \citep{wilson2013gaussian, samo2015generalized, remes2017non}. In the compositional approach kernels are composed of simpler kernels, whose choice often remains ad-hoc.
\par
The spectral representation approach proposed by \citet{quia2010sparse}, and extended by \citet{wilson2013gaussian}, constructs \emph{stationary} kernels as the Fourier transform of a Gaussian mixture, with theoretical support from the Bochner's theorem. Stationary kernels are unsuitable for large-scale datasets that are typically rife with locally-varying patterns \citep{samo2016string}. \citet{remes2017non} proposed a practical \emph{non-stationary} spectral kernel generalisation based on Gaussian process frequency functions, but with explicitly unclear theoretical foundations. An earlier technical report studied a non-stationary spectral kernel family derived via the generalised Fourier transform \citep{samo2015generalized}. \citet{samo2017advances} expanded the analysis into non-stationary continuous bounded kernels. \par 
The cubic time complexity of GP models significantly hinders their scalability. Sparse Gaussian process models \citep{herbrich2003fast, snelson2006sparse, titsias2009variational,hensman2015scalable} scale GP models with variational inference on pseudo-input points as a concise representation of the input data. Inter-domain Gaussian processes generalize sparse GP models by linearly transforming the original  GP and computing cross-covariances, thus putting the inducing points on the transformed domain \citep{lazaro2009inter}.
\begin{table*}[t]
    \centering
    \resizebox{\textwidth}{!}{
    \begin{tabular}{lcccr}
    Kernel & Harmonizable & Non-stationary & Spectral inference & Reference \\
    \hline
    SE: squared exponential & \cmark & \xmark & \cmark & \citet{rasmussen2006} \\
    SS: sparse spectral  & \cmark & \xmark &\cmark & \citet{quia2010sparse} \\
    SM: spectral mixture & \cmark & \xmark & \cmark & \citet{wilson2013gaussian} \\
    GSK: generalised spectral kernel  & \cmark & \cmark & \xmark &\citet{samo2017advances}\\
    GSM: generalised spectral mixture &\bf{?} & \cmark & \xmark &\citet{remes2017non} \\
    HMK: harmonizable mixture kernel & \cmark & \cmark & \cmark & current work
    \end{tabular}
    }
    \caption{Overview of proposed spectral kernels. The SE, SS and SM kernels are stationary with scalable spectral inference paradigms \citep{lazaro2009inter, quia2010sparse, gal2015improving}. The GSM kernel is theoretically poorly defined with unknown harmonizable properties. HMK is well-defined with variational Fourier features as spectral inference.}
    \label{tab:spkernels}
\end{table*}

In this paper we propose a theoretically sound treatment of non-stationary kernels, with main contributions:
\begin{itemize}
    \item We present a detailed analysis of \textit{harmonizability}, a concept mainly existent in statistics literature. Harmonizable kernels are non-stationary kernels interpretable with their \emph{generalized} spectral representations, similar to stationary ones.
    \item We propose practical \emph{harmonizable mixture kernels} (HMK), a class of kernels dense in the set of harmonizable covariances with a mixture generalized spectral distribution. 
    \item We propose \emph{variational Fourier features}, an inter-domain GP inference framework for GPs equipped with HMK. Functions drawn from such GP priors have a well-defined Fourier transform, a desirable property not found in stationary GPs.
\end{itemize}



\section{HARMONIZABLE KERNELS}

In this section we introduce \emph{harmonizability}, a generalization of stationarity largely unknown to the field of machine learning. We first define harmonizable kernel, and then analyze two existing special cases of harmonizable kernels, stationary and locally stationary kernels. We present a theorem demonstrating the expressiveness of previous stationary spectral kernels. Finally, we introduce Wigner transform as a tool to interpret and analyze these kernels.\par
Throughout the discussion in the paper, we consider complex-valued kernels with vectorial input $k(\x, \x'): \R^D\times \R^D\mapsto\mathbb{C}$, and we denote vectors from the input (data) domain with symbols $\x, \x', \btau, \bf{t}$, while we denote frequencies with symbols $\bxi, \bomega$.

\begin{figure*}[t]
    \centering
    \includegraphics[width=\textwidth]{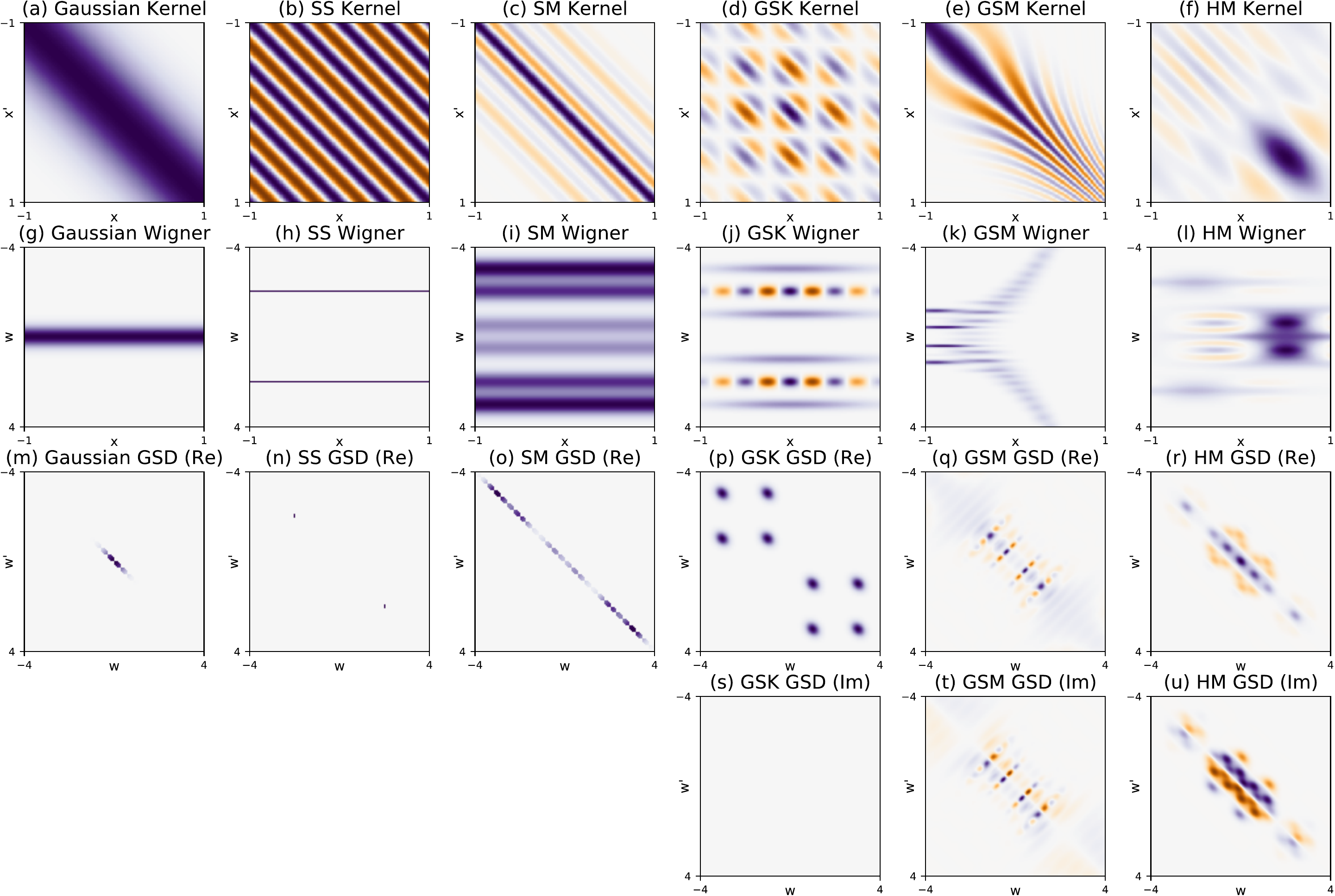}
    \caption{Comparison of Gaussian, SS, SM, GSK, GSM and HM kernels (columns) with respect to the kernel, Wigner distribution, and the generalized spectral density including real and imaginary part (rows).}
    \label{fig:fig1}
\end{figure*}

\subsection{Harmonizable kernel definition}

A harmonizable kernel \citep{kakihara1985, yaglom1987correlation, loeve1994probability} is a kernel with a \emph{generalized spectral distribution} defined by a generalized Fourier transform:
\begin{definition}
A complex-valued bounded continuous kernel $k: \mathbb{R}^D \times\mathbb{R}^D\mapsto \mathbb{C}$ is \emph{harmonizable} when it can be represented as
\begin{align}
    k(\x,\x') &= \int_{\mathbb{R}^D\times\mathbb{R}^D} e^{2i\pi(\bomega^\top \x-\bxi^\top \x')}\mu_{\Psi_k}(\text{d}\bomega, \text{d}\bxi),
\end{align}
where $\mu_{\Psi_k}$ is the Lebesgue-Stieltjes measure associated to some positive definite function $\Psi_k(\bomega, \bxi)$ with bounded variations.
\end{definition}

Harmonizability is a property shared by kernels and random processes with such kernels. The positive definite measure induced by function $\Psi_k$ is defined as the generalized spectral distribution of the kernel, and when $\mu_{\Psi_k}$ is twice differentiable, the derivative $S_k(\bomega, \bxi) = \dfrac{\partial^2\Psi_k}{\partial\bomega\partial\bxi}$ is defined as \emph{generalized spectral density} (GSD).\par
Harmonizable kernel is a very general class in the sense that it contains a large portion of bounded, continuous kernels (See Table \ref{tab:spkernels}) with only a handful of (somewhat pathological) exceptions \citep{yaglom1987correlation}.

\subsection{Comparison with Bochner's theorem}
Stationary kernels are kernels whose value only depends on the distance $\btau=\x-\x'$, and therefore is invariant to translation of the input. Bochner's theorem \citep{bochner1959lectures, stein2012interpolation} expresses similar relation between finite measures and kernels:
\begin{theorem}
(Bochner) A complex-valued function $k: \mathbb{R}^D\times\mathbb{R}^D\mapsto\mathbb{C}$ is the covariance function of a weakly stationary mean square continuous complex-valued random process on $\mathbb{R}^D$ if and only if it can be represented as \begin{align}
    k(\btau) &= \int_{\mathbb{R}^D} e^{2i\pi\bomega^\top\btau} \psi_k(\text{d}\bomega).
\end{align}
where $\psi_k$ is a positive finite measure.
\end{theorem}
Bochner's theorem draws duality between the space of finite measures to the space of stationary kernels. The \emph{spectral distribution} $\psi_k$ of a stationary kernel is the finite measure induced by a Fourier transform. And when $\psi_k$ is absolutely continuous with respect to the Lebesgue measure, its density is called \emph{spectral density} (SD),  $S_k(\bomega)=\dfrac{\d{\psi_k(\bomega)}}{\d{\bomega}}$.\par
Harmonizable kernels include stationary kernels as a special case. When the mass of the measure $\mu_\Psi$ is concentrated on the diagonal $\bomega=\bxi$, the generalized inverse Fourier transform devolves into an inverse Fourier transform with respect to $\btau=\x-\x'$, and therefore recovers the exact form in Bochner's theorem.

A key distinction between the two spectral distributions is that the spectral distribution is a nonnegative finite measure, but the generalized spectral distribution is a complex-valued measure with subsets assigned to complex numbers. Even with a real-valued harmonizable kernel, $\Psi_k$ can be complex-valued.

\subsection{Stationary spectral kernels}
The perspective of viewing the spectral distribution as a normalized probability measure makes it possible to construct expressive stationary kernels by modeling their spectral distributions. Notable examples include the sparse spectrum (SS) kernel \citep{quia2010sparse}, and spectral mixture (SM) kernel \citep{wilson2013gaussian},
\begin{align}
    k_{SS}(\btau) &= \sum_{q=1}^Q \alpha_q\cos(2\pi\bomega_q^\top\btau),\\
    k_{SM}(\btau) &= \sum_{q=1}^Q \alpha_qe^{-2\pi^2\tau^\top\bSigma_q\tau}\cos(2\pi\bomega_q^\top\btau),
\end{align}
with number of components $Q \in \mathbb{N}_+$, the component weights (amplitudes) $\alpha_q \in \mathbb{R}_+$, the (mean) frequencies $\bomega_q\in\R_+^D$, and the frequency covariances $\bSigma_q \succeq \mathbf{0}$.
Here we prove a theorem demonstrating the expressiveness of the above two kernels.
\begin{theorem}
Let $h$ be a complex-valued positive definite, continuous and integrable function. Then the family of \emph{generalized spectral kernels}
\begin{align}
    k_{GS}(\btau) &= \sum_{q=1}^Q \alpha_q h(\btau\circ\bgamma_q)e^{2i\pi\bomega_q^\top\btau},
\end{align}
is dense in the family of stationary, complex-valued kernels with respect to pointwise convergence of functions. Here $\circ$ denotes the Hadamard product, $\alpha_q\in\mathbb{R}_+$, $\bomega_k\in\mathbb{R}^D$, $\bgamma_k\in\mathbb{R}^{D}_+$, $Q\in\mathbb{N}_+$.
\end{theorem}
\begin{proofskch}
We know that discrete measures are dense in the Banach space of finite measures. Therefore, the complex extension of sparse spectrum kernel
    $k_{SS}(\btau) = \sum_{k=1}^K \alpha_k e^{2i\pi\bomega_k^\top\btau}$ is dense in stationary kernels.\par
For each $q$, the function $\dfrac{\alpha_q}{h(0)} h(\btau\circ\bgamma_q)e^{2i\pi\bomega_k^\top\btau}$ converges to $\alpha_q e^{2i\pi\bomega_q^\top\btau}$ pointwise as $\bgamma_q\downarrow \mathbf{0}$. Therefore, the proposed kernel form is dense in the set of sparse spectrum kernels, and by extension, stationary kernels.
See Section 1 in supplementary materials for a more detailed proof.
\end{proofskch}\par
We strengthen the claim of \citet{samo2015generalized} by adding a constraint $\alpha_k > 0$ that restricts the family of functions to only valid PSD kernels \citep{samo2017advances}. The spectral distribution of $k_{GS}$ is
\begin{align}
    \psi_{k_{GS}}(\bxi) &= \sum_{q=1}^Q \dfrac{\alpha_q}{\prod_{d=1}^D\gamma_{kd}}\psi_h((\bxi-\bomega_k)\oslash\bgamma_k),
\end{align}
with $\oslash$ denoting elementwise division of vectors. A real-valued kernel can be obtained by averaging a complex kernel with its complex conjugate, which induces a symmetry on the spectral distribution, $\psi_k(\bxi) = \psi_k(-\bxi)$. For instance, the SM kernel has the symmetric Gaussian mixture spectral distribution 
\begin{align}
    \psi_{k_{SM}}(\bxi) &= \dfrac{1}{2}\sum_{q=1}^Q\alpha_q(\N(\bxi|\bomega_q, \bSigma_q)+\N(\bxi|-\bomega_q, \bSigma_q)). 
\end{align}

\subsection{Locally stationary kernels}

As a generalization of stationary kernels, the locally stationary kernels \citep{silverman1957locally} are a simple yet unexplored concept in machine learning. A locally stationary kernel is a stationary kernel multiplied by a sliding power factor: 
\begin{align}
    k_{LS}(\x,\x') &= k_1\left(\dfrac{\x+\x'}{2}\right)k_2(\x-\x').
\end{align}
where $k_1: \mathbb{R}^D\mapsto\mathbb{R}_{\geq 0}$ is an arbitrary nonnegative function, and $k_2:\mathbb{R}^D\mapsto\mathbb{C}$ is a stationary kernel. $k_1$ is a function of the \emph{centroid} between $\x$ and $\x'$, describing the scale of covariance on a global structure, while $k_2$ as a stationary covariance describes the local structure \citep{genton2001classes}. It is straightforward to see that locally stationary kernels reduce into stationary kernels when $k_1$ is constant.

Integrable locally stationary kernels are of particular interest because they are harmonizable with a GSD. Consider a locally stationary Gaussian kernel (LSG) defined as a SE kernel multiplied by a Gaussian density on the centroid $\tx = (\x+\x')/2$. Its GSD can be obtained using the generalized Wiener-Khintchin relations \citep{silverman1957locally}.
\begin{align}
    k_{\text{LSG}}(\x, \x') &= e^{-2\pi^2\tx^\top\bSigma_1\tx}e^{-2\pi^2\btau^\top\bSigma_2\btau},\\
    S_{k_{\text{LSG}}}(\bomega, \bxi) &= \N\left(\left.\dfrac{\bomega+\bxi}{2}\right\vert 0, \bSigma_2\right)\N\left(\left.\bomega-\bxi\right\vert 0, \bSigma_1\right).
\end{align}

\subsection{Interpreting spectral kernels}

While the spectral distribution of a stationary kernel can be easily interpreted as a `spectrum', the analogy does not apply to harmonizable kernels. In this section, we introduce the Wigner transform \citep{flandrin1998time} which adds interpretability to kernels with spectral representations.
\begin{definition}
The \emph{Wigner distribution function} (WDF) of a kernel $k(\cdot,\cdot):\mathbb{R}^D\times\mathbb{R}^D\mapsto\mathbb{C}$ is defined as $W_k:\mathbb{R}^D\times\mathbb{R}^D\mapsto\mathbb{R}$:
\begin{align}
    W_k(\x, \bomega) &= \int_{\mathbb{R}^D} k\left(\x+\dfrac{\btau}{2}, \x-\dfrac{\btau}{2}\right)e^{-2i\pi\bomega^\top\btau} \d{\btau}.
\end{align}
\end{definition}

The Wigner transform first changes the kernel form $k$ into a function of the centroid of the input: $(\x+\x')/2$ and the lag $\x-\x'$, and then takes the Fourier transform of the lag. The Wigner distribution functions are fully equivalent to non-stationary kernels. Given the domain of WDF, we can view WDF as a `spectrogram' demonstrating the relation between input and frequency. Converting an arbitrary kernel into its Wigner distribution sheds light into the frequency structure of the kernel (See Figure \ref{fig:fig1}).\par
The WDFs of locally stationary kernels adhere to the intuitive notion of local stationarity where frequencies remain constant at a local scale. Take locally stationary Gaussian kernel $k_{\text{LSG}}$ as an example:
\begin{align}
    W_{k_{\text{LSG}}}(\x,\bomega) &= \N(\bomega|\mathbf{0}, \bSigma_2) e^{-2\pi^2\x^\top\bSigma_1\x}.
\end{align}

\section{HARMONIZABLE MIXTURE KERNEL}
In this section we propose a novel \emph{harmonizable mixture kernel}, a family of kernels dense in harmonizable covariance functions. We present the kernel in an intentionally concise form, and refer the reader to the Section 2 in the Supplements for a full derivation.

\subsection{Kernel form and spectral representations}

The \emph{harmonizable mixture kernel} (HMK) is defined with an additive structure:
\begin{align}
    k_{\text{HM}}(\x,\x')&=\sum_{p=1}^P k_p(\x-\x_p, \x'-\x_p),\\
    k_p(\x, \x') &= k_{\text{LSG}}(\x\circ\bgamma_p, \x'\circ\bgamma_p)\phi_p(\x)^\top\B_p\phi_p(\x'),
\end{align}
where $P\in\mathbb{N}_+$ is the number of centers, $\left(\phi_p(\x)\right)_{q=1}^{Q_p}=e^{2i\pi\bmu_{pq}^\top\x}$ are sinusoidal feature maps, $\B_p\succeq\mathbf{0}_{Q_p}$ are spectral amplitudes, $\bgamma_p\in\R^D_+$ are input scalings, $\x_p\in\R^D$ are input shifts, and $\bmu_{pq}\in\R^D$ are frequencies. It is easy to verify $k_{\text{HM}}$ as a valid kernel, for each $k_p$ is a product of an LSG kernel and an inner product with finite basis expansion of sinusoidal functions.\par

HMKs have closed form spectral representations such as \emph{generalized spectral density} (See Section 2 in the Supplement for detailed derivation):
\begin{align}
    S_{k_{\text{HM}}}(\bomega, \bxi) &= \sum_{p=1}^P S_{k_p}(\bomega, \bxi)e^{-2i\pi\x_p^\top(\bomega-\bxi)},\\
    S_{k_p}(\bomega, \bxi) &= \dfrac{1}{\prod_{d=1}^D\gamma_{pd}^2}\sum_{1\leq i, j \leq Q_p} b_{pij}S_{pij}(\bomega, \bxi),\\
    S_{pij}(\bomega, \bxi)&=S_{k_{\text{LSG}}}((\bomega-\bmu_{pi})\oslash\bgamma_p, (\bxi-\bmu_{pj})\oslash\bgamma_p).
\end{align}
The \emph{Wigner distribution function} can be obtained in a similar fashion
\begin{align}
    W_{k_{\text{HM}}}(\x,\bomega)&=\sum_{p=1}^P W_{k_p}(\x-\x_p, \bomega),\\
    W_{k_p}(\x,\omega) &= \dfrac{1}{\prod_{d=1}^D\gamma_{pd}}\sum_{1\leq i,j\leq Q_p} W_{pij}(\x,\bomega),\\
    W_{pij}(\x,\bomega) &= W_{k_{\text{LSG}}}\left(\x\circ\bgamma_p, \left(\bomega-(\bmu_{pi}+\bmu_{pj})/2\right)\oslash\bgamma_p\right)\notag\\
    &\times\cos(2\pi(\bmu_{pi}-\bmu_{pj})^\top\x).
\end{align}
The kernel form, GSD and WDF both take a normal density form. It is straightforward to see $S_{k_{\text{HM}}}$ is PSD, and that $k_{\text{Hm}}(-\x, -\x')$ is the GSD of $S_{k_{\text{HM}}}$. A real-valued kernel $k_r$ is obtained by averaging with its complex conjugate: $W_{k_r}(\x,\bomega)=W_{k_r}(\x,-\bomega)$, $S_{k_r}(\bomega, \bxi) = S_{k_r}(-\bomega, -\bxi)$.

\subsection{Expressiveness of HMK}
Similar to the construction of \emph{generalized spectral kernels}, we can construct a generalized version $k_h$ where $k_{\text{LSG}}$ is replaced by $k_{\text{LS}}$, a locally stationary kernel with a GSD.
\begin{theorem} Given a continuous, integrable kernel $k_{\text{LS}}$ with a valid \emph{generalized spectral density}, the harmonizable mixture kernel
\begin{align}
    k_h(\x,\x')&=\sum_{p=1}^P k_p(\x-\x_p, \x'-\x_p),\\
    k_p(\x, \x')&=k_{\text{LS}}(\x\circ\bgamma_p,\x'\circ\bgamma_p)\phi_p(\x)^\top\B_p\phi_p(\x'),
\end{align}
is dense in the family of harmonizable covariances with respect to pointwise convergence of functions. Here $P\in\mathbb{N}_+$, $(\phi_p(\x))_q=e^{2i\pi\bmu_{pq}^\top\x}$, $q=1,\hdots, Q_p$, $\bgamma_p\in\R_+^D$, $\x_p\in\R^D$, $\bmu_{pq}\in\R^D$, $\B_p$ as positive definite Hermitian matrices.
\end{theorem}
\begin{proof}
See Section 3 in the supplementary materials.
\end{proof}

\section{VARIATIONAL FOURIER FEATURES}

In this section we propose variational inference for the harmonizable kernels applied in Gaussian process models. 

We assume a dataset of $n$ input $X = \{\x_i\}_{i=1}^n$ and output $\y = \{ y_i\} \in \R^n$ observations from some function $f(\x)$ with a Gaussian observation model:
\begin{align}
y = f(\x) + \varepsilon, \qquad \varepsilon \sim \N(0, \sigma_y^2). \label{eq:noise}
\end{align}

\subsection{Gaussian processes}

Gaussian processes (GP) are a family of Bayesian models that characterise distributions of functions \citep{rasmussen2006}. We assume a zero-mean Gaussian process prior on a latent function $f(\x) \in \R$ over vector inputs $\x \in \R^D$
\begin{align}
f(\x) &\sim \GP( 0, K(\x,\x')),
\end{align}
which defines a priori distribution over function values $f(\x)$ with mean $\mathbb{E}[ f(\x)] = 0$ and covariance 
\begin{align}
\cov[ f(\x), f(\x')] &= K(\x,\x').
\end{align}
A GP prior specifies that for any collection of $n$ inputs $X$, the corresponding function values $\f = ( f(\x_1), \ldots, f(\x_n))^\top \in \R^n$ are coupled by following a multivariate normal distribution 
$\f \sim \N(\0, \K_{ff}),$
where $\K_{ff} = (K(\x_i, \x_j))_{i,j=1}^n \in \R^{n \times n}$ is the kernel matrix over input pairs. The key property of GP's is that output predictions $f(\x)$ and $f(\x')$ correlate according to how similar are their inputs $\x$ and $\x'$ as defined by the kernel $K(\x,\x') \in \R$. 



\subsection{Variational inference with inducing features}


In this section, we introduce variational inference of sparse GPs in an inter-domain setting. Consider a GP prior $f(\x)\sim\mathcal{GP}(0, k)$, and a valid linear transform $\mathcal{L}$ projecting $f$ to another GP $\mathcal{L}_f(\z)\sim\mathcal{GP}(0, k')$. 


We begin by \emph{augmenting} the Gaussian process with $m < n$ inducing variables $u_j = \mathcal{L}_f(\z_j)$ using a Gaussian model. $\z_j$ are \emph{inducing features} placed on the domain of $\mathcal{L}_f(\z)$, with prior $p(\u) = \N( \u | \0, \K_{uu})$ and a conditional model \citep{hensman2015scalable}
\begin{align}
    p(\f | \u) &= \N( \A \u, \K_{ff} - \A \K_{uu} \A^\dag), \label{eq:interp}
\end{align}
where $\A = \K_{fu} \K_{uu}^{-1}$, and $\A^\dag$ denotes the Hermitian transpose of $\A$ allowing for complex GPs. 
The kernel $\K_{uu}$ is between the $m \times m$ inducing variables 
and the kernel $\K_{fu}$ is the cross covariance of $\mathcal{L}$, $\left(\K_{fu}\right)_{is} = \cov(f(\x_i), \mathcal{L}_f(\z_s))$. Next, we define a variational approximation $q(\u) = \N( \u | \mathbf{m}, \mathbf{S})$ with the Gaussian interpolation model \eqref{eq:interp},
\begin{align}
    q(\f) 
     &= \N( \f | \A \mathbf{m}, \K_{ff} - \A (\mathbf{S} - \K_{uu}) \A^\dag),
\end{align}
with free variational mean $\mathbf{m} \in \R^m$ and variational covariance $\mathbf{S} \in \R^{m \times m}$ to be optimised. Finally, variational inference \citep{blei2016} describes an evidence lower bound (ELBO) of augmented Gaussian processes as
\begin{align}
    \hspace{-2.5mm}\log p(\y) & \ge \sum_{i=1}^n \mathbb{E}_{q(f_i)}  \log p(y_i | f_i)  - \mathrm{KL}[ q(\u) || p(\u)]. \label{eq:elbo}
\end{align}

\subsection{Fourier transform of a harmonizable GP}


In this section, we compute cross-covariances between a GP and the Fourier transform of the GP. Consider a GP prior $f\sim\mathcal{GP}(0,k)$ where the kernel $k$ is harmonizable with a GSD $S_k$ and where $\hat{f}$ is the Fourier transform of $f$, 
\begin{align}
    \hat{f}(\bomega) &\triangleq \int_{\mathbb{R}^D} f(\x)e^{-2i\pi\bomega^\top \x}\d{\x}.
\end{align}
The validity of this setting is easily verified because $f$ is square integrable on expectation,
\begin{align}
    \mathbb{E}\left\{\int_{\R^D} |f(\x)|^2\d{\x}\right\} &= \int_{\R^D} k(\x,\x) \d{\x} < \infty.
\end{align}

We can therefore derive the cross-covariances
\begin{align}
    \cov(\hat{f}(\bomega), f(\x)) 
    &= \int_{\R^D} k(\t,\x) e^{-2 i \pi\bomega^\top \t} \d{\t} \\
    \cov(\hat{f}(\bomega), \hat{f}(\bxi)) 
    &= S_k(\bomega, \bxi).
\end{align}
The above derivation is valid for any harmonizable kernel with a GSD. The Fourier transform of $\mathcal{GP}(0,k)$ is a complex-valued GP with kernel $S_k$, which correlates to the original GP.\par
For harmonizable, integrable kernel $k$, we can construct an inter-domain sparse GP model defined in 4.2 by plugging in $\mathcal{L}_f = \hat{f}$.

\begin{figure*}[t]
\begin{center}
    \includegraphics[width=\textwidth, clip=true]{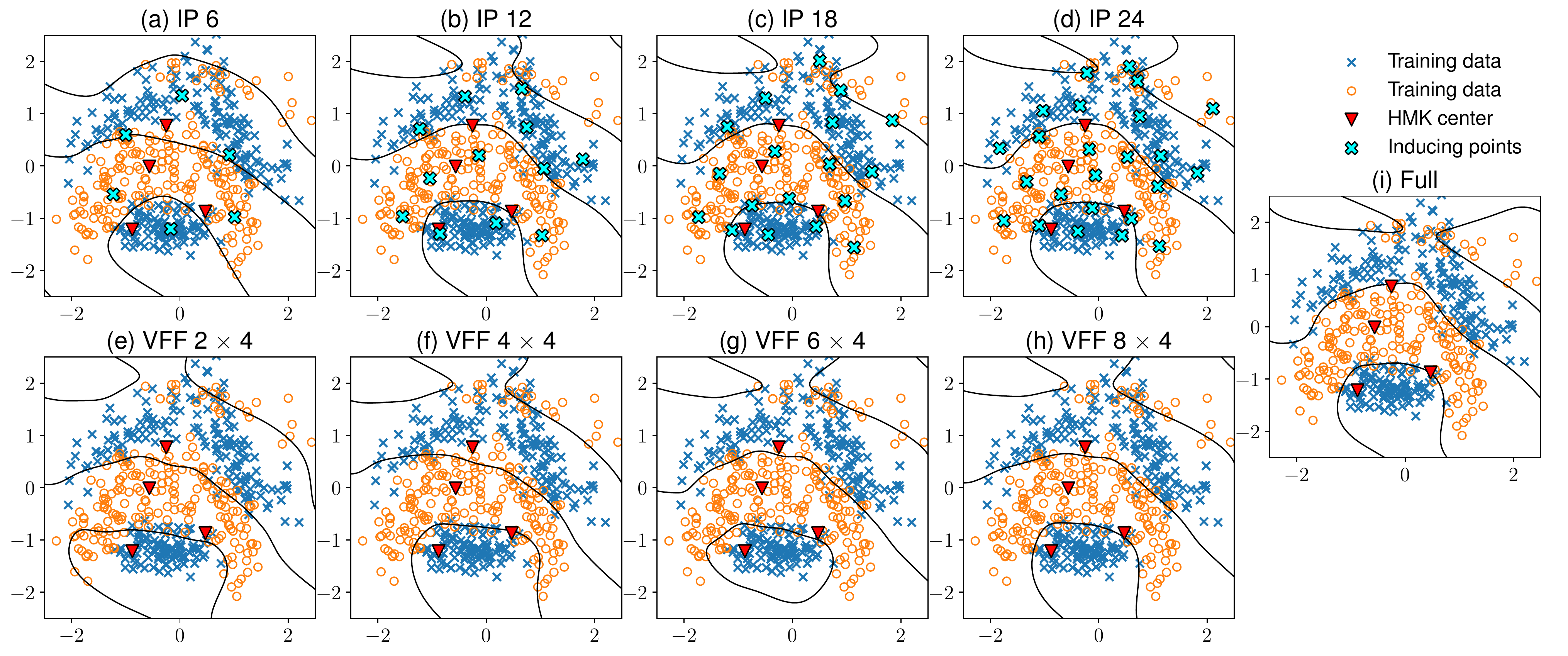}
\end{center}
    \caption{Sparse GP classification with the banana dataset. The model is learned by an HMK with $P=4$ components, and thus 2 inducing frequencies for each component constitute a total of $2 \times 4$ inducing frequencies.}
  \label{figure:gpc}
\end{figure*}

\subsection{Variational Fourier features of the harmonizable mixture kernel}
HMK belongs to the kernel family discussed in 4.3, but we can utilize the additive structure of an HMK $k_{HM} = \sum_{p=1}^P k_p(\x-\x_p, \x'-\x_p)$. A GP with kernel $k_{HM}$ can be decomposed into $P$ independent GPs:
\begin{align}
    f(\x) &= \sum_{p=1}^P f_p(\x-\x_p),\\
    f_p(\x) &\sim \mathcal{GP}(0, k_p(\x, \x')).
\end{align}
Given this formulation, we can derive \emph{variational Fourier features} with inducing frequencies conditioned on one $f_p$. For the $p^{th}$ component, we have $m_p$ inducing frequencies $(\bomega_{p1}, \ldots, \bomega_{pm_p})$ and $m_p$ inducing values $(u_{p1}, \cdots, u_{pm_p})$. We can compute inter-domain covariances in a similar fashion:
\begin{align}
    \K_{fu}(\bomega_{qj}, \x) &\triangleq \cov(f(x), u_{qj}) \label{eq:kfu} \\
    &= \sum_{p=1}^P\cov(f_p(\x-\x_p), u_{qj}) \notag \\
    &= \cov(f_q(\x-\x_q), \hat{f}_q(\bomega_{qj})). \notag
\end{align}
Similarly, we compute entries of the matrix $K_{uu}$
\begin{align}
    \K_{uu}(\bomega_{pi}, \bomega_{qj}) \triangleq \cov(u_{pi}, u_{qj}) &= \begin{cases}
    S_p(\bomega_{pi}, \bomega_{qj}), p=q,\\
    0, p\neq q.
    \end{cases} \label{eq:kuu}
\end{align}
The matrix $\K_{uu}$ allows for a block diagonal structure, which allows for faster matrix inversion. The variational Fourier features are then completed by plugging in entries in $\K_{fu}$ \eqref{eq:kfu} and $\K_{uu}$ \eqref{eq:kuu} into the evidence lower bound \eqref{eq:elbo}.\par
\subsection{Connection to previous work}
In this section we demonstrate that an inter-domain stationary GP with windowed Fourier transform \citep{lazaro2009inter} is equivalent to a rescaled VFF with a tweaked kernel. GPs with stationary kernels do not have valid Fourier transform, therefore, previous attempts of using Fourier transforms of GPs have been accompanied by a window function:
\begin{align}
    \mathcal{L}_f(\bomega) &= \int_{\mathbb{R}^D} f(\x) w(\x) e^{-2i\pi\bomega^\top \x} \d{\x}.
\end{align}
The windowing function $w(\x)$ can be a soft Gaussian window $w(\x) = \mathcal{N}(\x|\bmu,\bSigma)$ \citep{lazaro2009inter} or a hard interval window $w(x)=\mathbb{I}_{[a\leq x\leq b]}e^{2i\pi a}$ \citep{hensman2017variational}. The windowing approach shares the caveat of a blurred version of the frequency space, caused by an inaccurate Fourier transform\citep{lazaro2009inter}.\par
Consider $f\sim\mathcal{GP}(0, k)$ where $k$ is a stationary kernel, and $w(\x) = \mathcal{N}(x|\mu,\bSigma)$, we see that $g(\x) = w(\x)f(\x) \sim\mathcal{GP}(0, w(\x)w(\x')k(\x-\x'))$. It is easy to verify that the kernel of $g(\x)$ is locally stationary. There exist the following relations of cross-covariances:
\begin{align}
    \cov(f(\x), \mathcal{L}_f(\bomega)) &= \dfrac{\cov(g(\x), \hat{g}(\bomega))}{w(\x)},\\
    \cov(\mathcal{L}_f(\bomega), \mathcal{L}_f(\bxi)) &= \cov(\hat{g}(\bomega), \hat{g}(\bxi)).
\end{align}
Therefore, windowed inter-domain GPs are equivalent to rescaled GPs with a tweaked kernel.
\section{EXPERIMENTS}
In this section, we experiment with the harmonizable mixture kernels for kernel recovery, GP classification and regression. We use a simplied version of the harmonizable kernel where the two matrices of the locally stationary $k_{\text{LSG}}$ are diagonals: $\bSigma_1=\mbox{diag}(\sigma_d^2)$, $\bSigma_2=\lambda^2 I$. See Section 6 in the supplement for more detailed information.
\subsection{Kernel recovery}
We demonstrate the expressiveness of HMK by using it to recover certain non-stationary kernels. We choose the non-stationary \emph{generalized spectral mixture kernel} (GSM) \citep{remes2017non} and the covariance function of a time-inverted fractional Brownian motion (IFBM):

\resizebox{1.00\columnwidth}{!}{
\begin{tabular}{l}
    $k_{\text{GSM}}(x,x') = w(x)w(x') k_{\text{Gibbs}}(x, x')\cos(2\pi(\mu(x)x-\mu(x')x')), $ \\
    $k_{\text{Gibbs}}(x, x') = \sqrt{\dfrac{2l(x)l(x')}{l(x)^2+l(x')^2}}\exp\left(-\dfrac{(x-x')^2}{l(x)^2+l(x')^2}\right), $ \\
    $k_{\text{IFBM}}(t,s) = \dfrac{1}{2}\left(\dfrac{1}{t^{2h}}+\dfrac{1}{s^{2h}} - \left\vert\dfrac{1}{t}-\dfrac{1}{s}\right\vert^{2h}\right),$ \\
\end{tabular}
}

where $s, t \in (0.1, 1.1]$ and $x, x' \in [-1, 1]$. 
The hyperparameters of $k_{\text{HM}}$ are randomly initialized, and optimized with stochastic gradient descent.

\begin{figure}[t]
    \centering
    \includegraphics[width=\columnwidth, clip=true, ]{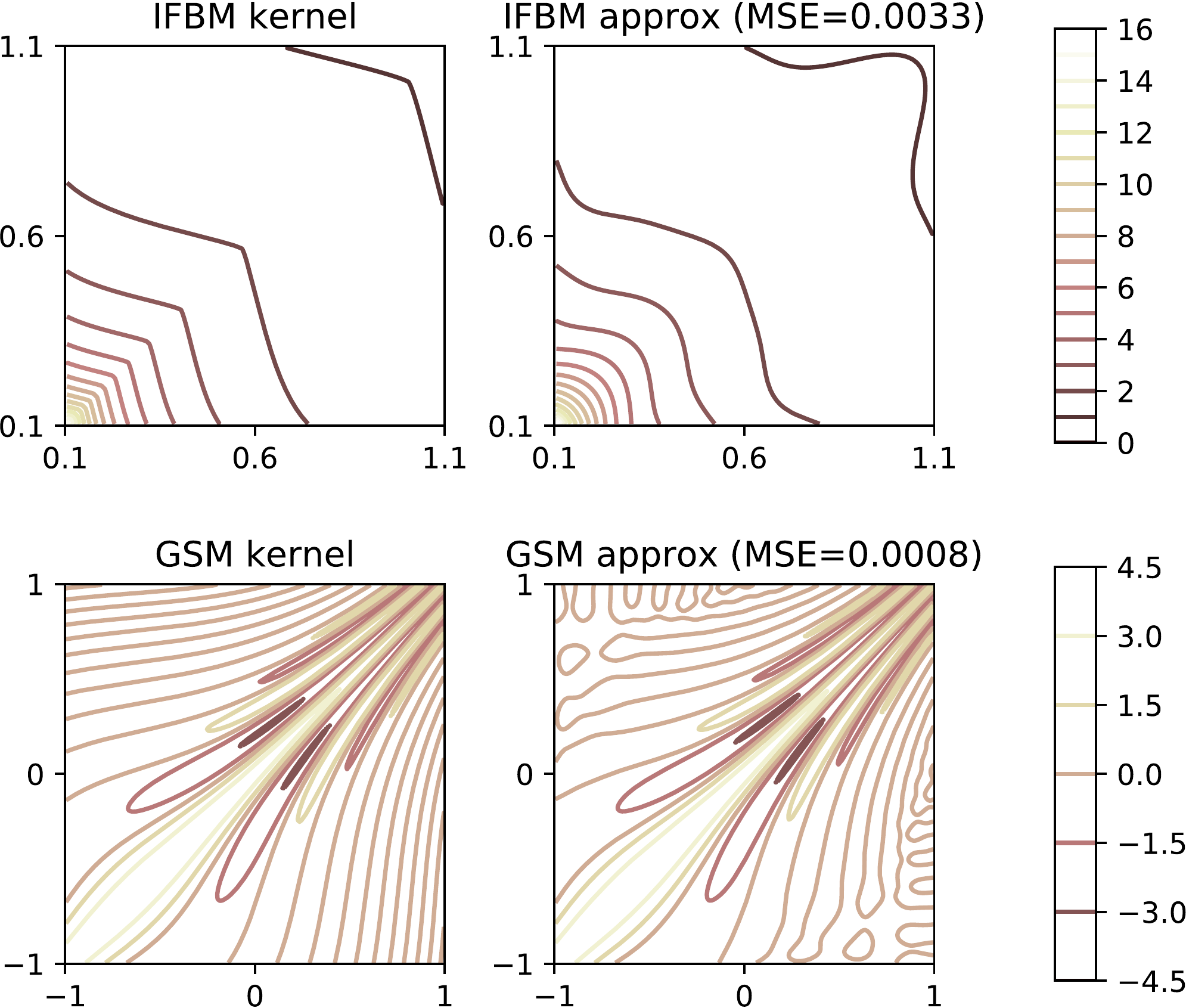}
    \caption{Kernel recovery experiment with true kernels (left) against SM kernel approximations (right).}
    \label{fig:ifbm_rec}
\end{figure}

Both kernels can be recovered almost perfectly with mean squared errors of $0.0033$ and $0.0008$. The result indicates that we can use the GSD and the Wigner distribution of the approximating HM kernel to interpret the GSM kernel (see Section 5 in supplementary materials). 

\begin{figure}[t]
  \centering
    \includegraphics[width=\columnwidth, clip=true]{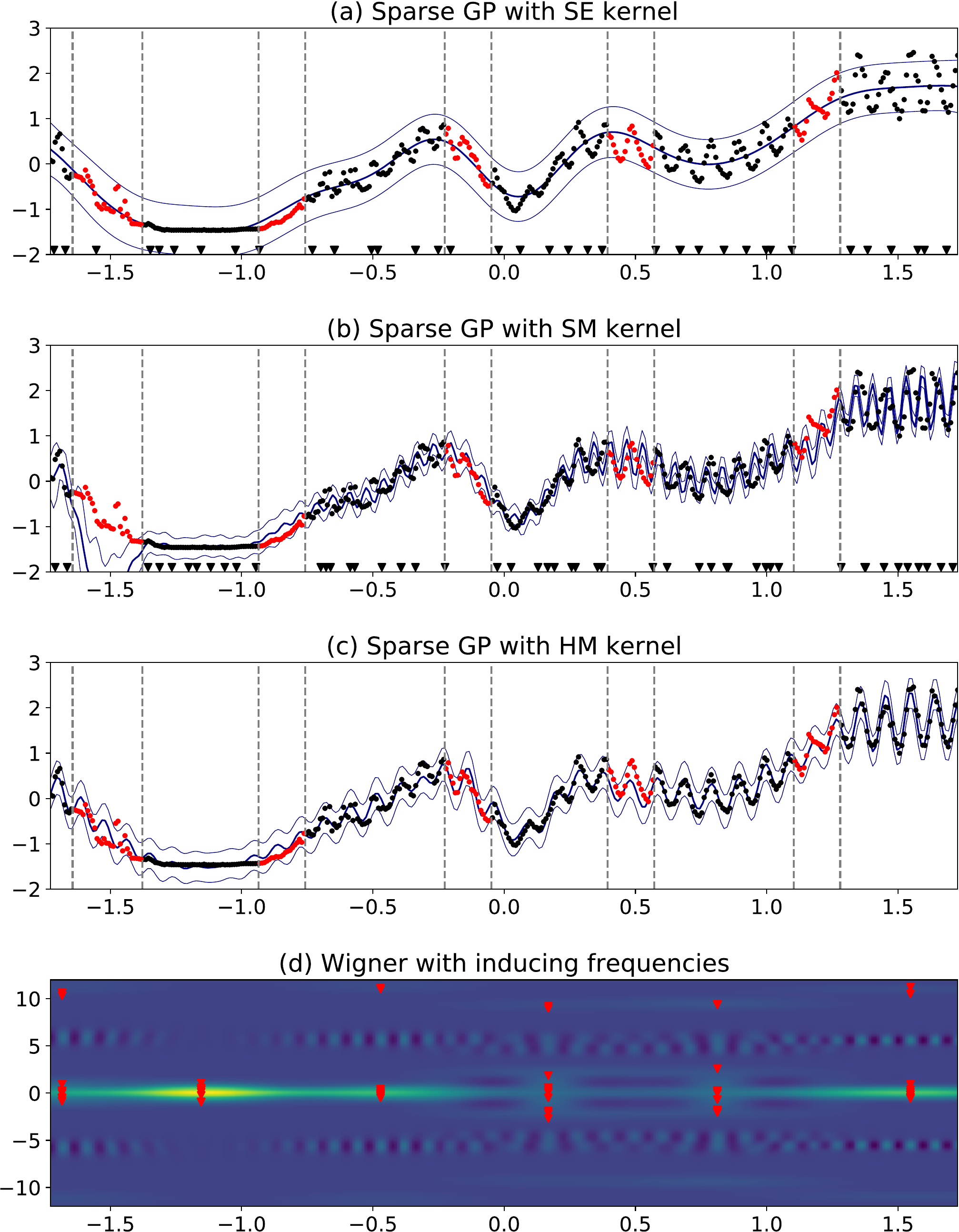}
    \caption{Sparse GP regression with solar irradiance dataset.}
  \label{figure:gpr}
\end{figure}

\subsection{GP classification with banana dataset}

In this section, we show the effectiveness of variational Fourier fetures in GP classification with HMK. We use an HMK with $P=4$ components to classify the banana dataset, and compare SVGP with inducing points (IP) \citep{hensman2015scalable} and SVGP with variational Fourier features (VFF). The model parameters are learned by alternating optimization rounds of natural gradients for the variational parameters, and Adam optimizer for the other parameters \citep{salimbeni2018natural}.

Figure \ref{figure:gpc} shows the decision boundaries of the two methods over the number of inducing points. 
For both variants, we experiment with model complexities from 6 to 24 inducing points in IP, and from 2 to 8 inducing frequencies for each component of HMK in the VFF. The centers of HMK (red triangles) spread to support the data distribution. The IP method is slightly more complex compared to VFF at the same parameter counts in terms of nonzero entries in the variational parameters.

The VFF method recovers roughly the correct decision boundary even with a small number of inducing frequencies, while converging faster to the decision boundaries as the number of inducing frequencies increases.

\subsection{GP regression with solar irradiance}

In this section, we demonstrate the effectiveness of HMK in interpolation for the non-stationary solar irradiance dataset. We run sparse GP regression with squared exponential, spectral mixture and harmonizable mixture kernels, and show the predicted mean, and 95\% confidence intervals for each model (See Figure \ref{figure:gpc}).

We use sparse GP regression proposed in \citep{titsias2009variational} with 50 inducing points marked at the x axis. The SE kernel can not estimate the periodic pattern and overestimates the signal smoothness. The SM kernel fits the training data well, but misidentifies frequencies on the first and fourth interval of the test set.


For sparse GP with HMK, we use the same framework where the variational lower bound is adjusted for VFF. The model extrapolates better for the added flexibility of nonstationarity, and the inducing frequencies aggregate near the learned frequencies. Both first and last test intervals are well fitted. The Wigner distribution with inducing frequencies of the optimised HM kernel is shown in Figure \ref{figure:gpc}d.



%


\section{CONCLUSION}
In this paper, we extend the generalization of Gaussian processes by proposing harmonizable mixture kernel, a non-stationary kernel spanning the wide class of harmonizable covariances. Such kernels can be used as an expressive tool for GP models. We also proposed variational Fourier features, an inter-domain inference framework used as drop-in replacements for sparse GPs. This work bridges previous research on spectral representation of kernels and sparse Gaussian processes.\par
Despite its expressiveness, one may brand the parametric form of HMK as not fully Bayesian, since it contradicts the nonparametric nature of GPs. A fully Bayesian approach would be to place a nonparametric prior over harmonizable mixture kernels, representing the uncertainty of the kernel form \citep{shah2014student}. 

\bibliographystyle{plainnat}
\bibliography{refs}

\newpage
\onecolumn
\section*{Supplementary materials}

\section*{1 \: \: Proof of theorem 2}
In this section, we prove the expressiveness of stationary spectral kernels.
\begin{theorem}
Let $h$ be a complex-valued positive definite, continuous and integrable function. Than the family of \emph{generalized spectral kernels}
\begin{align}
    k_{GS}(\btau) &= \sum_{q=1}^Q \alpha_q h(\btau\circ\bgamma_q)e^{2i\pi\bomega_q^\top\btau}.
\end{align}
with $\circ$ denoting the Hadamard product, $\alpha_q\in\mathbb{R}_+$, $\bomega_k\in\mathbb{R}^D$, $\bgamma_k\in\mathbb{R}^{D}_+$, $Q\in\mathbb{N}_+$ is dense in the family of stationary, complex-valued kernels with respect to pointwise convergence of functions.
\end{theorem}
\begin{proof}
We know from the uniform convergence of random Fourier features \citep{rahimi2008random}, that for an arbitrary stationary kernel $k_0(\x,\x') = k_0(\x-\x')$, for all compact subset $\mathcal{M}\in\R^D$, and for all $\epsilon > 0$, there exists a feature map $\zeta_{\omega} (\x) = \left(\alpha_qe^{2\pi\bomega_q^\top\x}\right)_{q=1}^Q$, such that $|\zeta_{\omega} (\x)\zeta_{\omega} (\x')^* - k_0(\x-\x')|<\epsilon$. The uniform convergence of random Fourier features suggests the expressiveness of a generalized form of sparse spectrum kernel $k_{\text{SS}}(\x-\x') = \sum_{q=1}^Q \alpha_qe^{2\pi\bomega_q^\top(\x-\x')}$.\par
For an arbitrary continuous, integrable kernel $h$, consider the function $\tilde{k}(\btau)=\dfrac{h(\btau\circ\bgamma)}{h(\bm{0})}k_{\text{SS}}(\btau), \bgamma \succeq \bm{0}$. Because of the continuity of function $h$, $\tilde{k}$ uniformly approximates $k_{\text{SS}}$ as $\bgamma \downarrow \bm{0}$, and thus can be used to approximate any stationary covariance $k_0$.\par
$\tilde{k}(\btau)$ uniformly approximates any stationary kernel $k_0$ on arbitrary compact subset $\mathcal{M}$ of $\R^D$. We can therefore construct a sequence of $\tilde{k}_n$ by setting $\epsilon_n = \dfrac{1}{n}$, $\mathcal{M}_n = \mathcal{B}(0, n) = \{v| \norm{v} \leq n \}$, $n=1, 2, 3, \cdots$. $\{\tilde{k}_n\}_{n=1}^{\infty}$ converges pointwise to $k_0$. $k_{\text{GS}}$ takes a more general form, and thus has the same level of expressiveness as $\tilde{k}$.
\end{proof}
We can see from the reasoning that sparse spectrum kernel and spectral mixture kernel both weakly span stationary covariances, and thus sharing the same level of expressiveness. But the sparse spectrum kernel only encodes a finite dimensional feature mapping, which reduces a GP regression with a sparse spectrum kernel to a Bayesian linear regression with trigonometric basis expansions. The spectral mixture kernel alleviates overfitting by using Gaussian mixture on the spectral distribution, which implicitly assumes certain level of smoothness of the unknown spectral distribution being modeled -- the Gaussian mixture also leads to an infinite-dimensional feature mapping which does not render a GP regression degenerate.\par
\section*{2 \: \: Derivation of harmonizable mixture kernel}
In this section we derive the parametric form of hramonizable mixture kernel.
The GSD of a locally stationary Gaussian kernel follows a generalized Wiener-Khintchin relation, as noticed in \citep{silverman1957locally}. This relation is easily noticed when subtituting $\x$ and $\x'$ with new variables $\tx=(\x+\x')/2$ and $\btau=\x-\x'$.
\begin{align}
    k_{\text{LSG}}(\x, \x') &= e^{-2\pi^2\tx^\top\bSigma_1\tx}e^{-2\pi^2\btau^\top\bSigma_2\btau},\\
    S_{k_{\text{LSG}}}(\bomega, \bxi) &= \iint k_{\text{LSG}}(\x, \x') e^{-2i\pi(\bomega^\top\x-\bxi^\top\x')} \d{\x}\d{\x'}\\
    &= \iint e^{-2\pi^2\tx^\top\bSigma_1\tx -2i\pi(\bomega-\bxi)^\top\tx}e^{-2\pi^2\btau^\top\bSigma_2\btau-i\pi(\bomega+\bxi)^\top\btau} \d{\tx}\d{\btau}\\
    &= \int e^{-2\pi^2\tx^\top\bSigma_1\tx -2i\pi(\bomega-\bxi)^\top\tx} \d{\tx}\int e^{-2\pi^2\btau^\top\bSigma_2\btau-i\pi(\bomega+\bxi)^\top\btau} \d{\btau}\\
    &= \N\left(\left.\bomega-\bxi\right\vert0, \bSigma_1\right)\N\left(\left.\dfrac{\bomega+\bxi}{2}\right\vert0, \bSigma_2\right).
\end{align}
The Wigner transform of $k_{\text{LSG}}$ is straightforward as the kernel factors into two parts.
\begin{align}
    W_{k_{\text{LSG}}}(\x, \bomega) &= \int k\left(\x+\dfrac{\btau}{2}, \x-\dfrac{\btau}{2}\right)e^{-2i\pi\btau^\top\bomega}\\
    &= e^{-2\pi^2\x^\top\bSigma_1\x} \int e^{-2\pi^2\btau^\top\bSigma_2\btau-2i\pi\btau^\top\bomega} \d{\btau}\\
    &= e^{-2\pi^2\x^\top\bSigma_1\x}\N(\left.\bomega\right\vert0, \bSigma_2).
\end{align}
Now consider the \emph{harmonizable mixture kernel},
\begin{align}
    k_{\text{HM}}(\x,\x')&=\sum_{p=1}^P k_p(\x-\x_p, \x'-\x_p),\\
    k_p(\x, \x') &= k_{\text{LSG}}(\x\circ\bgamma_p, \x'\circ\bgamma_p)\phi_p(\x)^\top\B_p\phi_p(\x')\\
    &= k_{\text{LSG}}(\x\circ\bgamma_p, \x'\circ\bgamma_p)\sum_{1\leq i,j \leq Q_p} e^{2i\pi(\bmu_{pi}^\top\x-\bmu_{pj}^\top\x')}.
\end{align}
We know from the Fourier transform $\hat{f}(\bxi)=\int f(\x) e^{-2i\pi\x^\top\bxi}\d{\x}$, that the translation in the input leads to closed form Fourier transforms: for $g(\x) = f(\x\circ\bgamma)$, $\hat{g}(\bxi)=\dfrac{1}{\prod\gamma_d}\hat{f}(\bxi\oslash\bgamma)$, and for $h(\x) = f(\x-\x_0)$, $\hat{h}(\bxi) =\hat{f}(\bxi) e^{-2i\pi\bxi^\top\x_0}$. The generalized Fourier transform to obtain GSD is equivalent to a Fourier transform of the concatenated vector $\left(\begin{matrix}\x \\ -\x'\end{matrix}\right)$. Using the above observations, we can obtain the GSD of the harmonizable mixture kernel.
\begin{align}
    S_{k_{\text{HM}}}(\bomega, \bxi) &= \sum_{p=1}^P S_{k_p}(\bomega, \bxi)e^{-2i\pi\x_p^\top(\bomega-\bxi)},\\
    S_{k_p}(\bomega, \bxi) &= \dfrac{1}{\prod_{d=1}^D\gamma_{pd}^2}\sum_{1\leq i, j \leq Q_p} b_{pij}S_{pij}(\bomega, \bxi),\\
    S_{pij}(\bomega, \bxi)&=S_{k_{\text{LSG}}}((\bomega-\bmu_{pi})\oslash\bgamma_p, (\bxi-\bmu_{pj})\oslash\bgamma_p).
\end{align}

The Wigner transform of a $k_{\text{HM}}$ requires an additional step of reverting the subscript.
\begin{align}
    k_p(\x, \x') &= k_{\text{LSG}}(\x\circ\bgamma_p, \x'\circ\bgamma_p)\sum_{1\leq i,j \leq Q_p}\beta_{pij} e^{2i\pi(\bmu_{pi}^\top\x-\bmu_{pj}^\top\x')}\\
    &= \dfrac{1}{2}k_{\text{LSG}}(\x\circ\bgamma_p, \x'\circ\bgamma_p)\sum_{1\leq i,j \leq Q_p}\beta_{pij}\left( e^{2i\pi(\bmu_{pi}^\top\x-\bmu_{pj}^\top\x')}+e^{2i\pi(\bmu_{pj}^\top\x-\bmu_{pi}^\top\x')}\right)\\
    &= k_{\text{LSG}}(\x\circ\bgamma_p, \x'\circ\bgamma_p) \sum_{1\leq i,j \leq Q_p}\beta_{pij}\left( \cos\left(2\pi\left(\dfrac{\bmu_{pi}+\bmu_{pj}}{2}\right)^\top\btau\right)\cos(2\pi(\bmu_{pi}-\bmu_{pj})^\top\tx) + ig(\tx,\btau)\right).
\end{align}
The imaginary part $g(\tx, \btau)$ is an odd function with respect to $\btau$: $g(\tx, \btau) = -g(\tx,-\btau)$, and thus has an integral of $0$ with Wigner transform. The above derivation gives a separable kernel formulation with respect to $\tx$ and $\btau$
\begin{align}
    W_{k_{\text{HM}}}(\x,\bomega)&=\sum_{p=1}^P W_{k_p}(\x-\x_p, \bomega),\\
    W_{k_p}(\x,\bomega) &= \dfrac{1}{\prod_{d=1}^D\gamma_{pd}}\sum_{1\leq i,j\leq Q_p} W_{pij}(\x,\bomega),\\
    W_{pij}(\x,\bomega) &= W_{k_{\text{LSG}}}\left(\x\circ\bgamma_p, \left(\bomega-(\bmu_{pi}+\bmu_{pj})/2\right)\oslash\bgamma_p\right)\cos(2\pi(\bmu_{pi}-\bmu_{pj})^\top\x).
\end{align}

\subsection*{2.1 \: \: Derivation of variational Fourier features}
For a GP with an integrable harmonizable kernel $k$, we can derive the cross-covariances between the primary GP $f$ and its Fourier transform $\hat{f}$:
\begin{align}
    \cov(\hat{f}(\bomega), f(\x)) &= \mathbb{E}\left\{\int f(\mathbf{t})f(\x) e^{-2i\pi\bomega^\top \mathbf{t}}\d{\mathbf{t}}\right\} \notag\\
    &= \int_{\R^D} k(\t,\x) e^{-2 i \pi\bomega^\top \t} \d{\t} \\
    \cov(f(\x), \hat{f}(\bomega)) &= \mbox{cov}(\hat{f}(\bomega), f(\x))^* \notag \\
    \cov(\hat{f}(\bomega), \hat{f}(\bxi)) &= \mathbb{E}\left\{\iint f(\x)f(\x')e^{-2i\pi(\bomega^\top \x - \bxi^\top \x')}\d{\x}\d{\x'}\right\} \notag \\
    &= \iint k(\x,\x') e^{-2i\pi(\bomega^\top \x - \bxi^\top \x')}\d{\x}\d{\x'}\notag \\
    &= S_k(\bomega, \bxi).
\end{align}
In the case of harmonizable mixture kernels, we need to compute closed form $\int k_p(\t, \x) e^{-2i\pi\bxi^\top\t} \d{\t}$ for the cross-covariances in variational Fourier features which is derived below:
\begin{align}
    \int k_p(\t, \x) e^{-2i\pi\bxi^\top\t} \d{\t} &= \sum_{1\leq i,j \leq Q_p} \beta_{pij}\exp\left(-2\pi^2\x^\top\left(\dfrac{\bSigma_1}{4}+\bSigma_2\right)-2i\pi\bmu_{pj}^\top\x\right)\notag \\
    &\times \int \exp\left(-2\pi^2(\t-\x_0)^\top\left(\dfrac{\bSigma_1}{4}+\bSigma_2\right)(\t-\x_0) + 2i\pi\mu_{pi}^\top \x -2i\pi\bxi^\top\x\right)\d{\x}\\
    &= \sum_{1\leq i,j \leq Q_p} \beta_{pij}\exp\left(-2\pi^2\x^\top\left(\dfrac{\bSigma_1}{4}+\bSigma_2\right)-2i\pi\bmu_{pj}^\top\x-2i\pi\x_0^\top\bxi\right)\notag \\
    &\times \N\left((\bxi-\bmu_{pi})\oslash\gamma_p\left\vert0, \dfrac{\bSigma_1}{4}+\bSigma_2\right.\right),\\
    \x_0 &= (\bSigma_1+4\bSigma_2)^{-1}(4\bSigma_2-\bSigma_1)\x.
\end{align}
\section*{3 \: \: Proof of theorem 3}
\begin{theorem} Given a continuous, integrable kernel $k_{\text{LS}}$ with a valid \emph{generalized spectral density}, the harmonizable mixture kernel
\begin{align}
    k_h(\x,\x')&=\sum_{p=1}^P k_p(\x-\x_p, \x'-\x_p),\\
    k_p(\x, \x')&=k_{\text{LS}}(\x\circ\bgamma_p,\x'\circ\bgamma_p)\phi_p(\x)^\dag\B_p\phi_p(\x'),
\end{align}
where $P\in\mathbb{N}_+$, $(\phi_p(\x))_q=e^{2i\pi\bmu_{pq}^\top\x}$, $q=1,\hdots, Q_p$, $\bgamma_p\in\R_+^D$, $\x_p\in\R^D$, $\bmu_{pq}\in\R^D$, $B_p$ as positive definite Hermitian matrices, is dense in the family of harmonizable covariances with respect to pointwise convergence of functions.
\end{theorem}

\begin{proof}
Discrete measures are dense in the Banach space of complex-valued measures on $\R^D\times\R^D$. And the same can be extended to the denseness of discrete positive definite bimeasures (a subset of measures on $\R^D\times\R^D$) in positive definite bimeasures. Intuitively, a harmonizable kernel $k_0: \R\times\R\mapsto \mathbb{C}$ with a generalized spectral density $S(\bomega, \bxi) =\dfrac{\partial^2\Psi(\bomega, \bxi)}{\partial\bomega\partial\bxi}$ can be expressed in the following form:
\begin{align}
    k_0(\x, \x') &= \iint S(\bomega,\bxi)e^{2i\pi(\bomega^\top\x-\bxi^\top\x')} \d{\bomega}\d{\bxi}.
\end{align}
Consider the Darboux sum with respect to a grid of frequencies $\bomega_0 < \bomega_2 < \hdots < \bomega_Q$
\begin{align}
    \sum_{1\leq u,v\leq Q} e^{2i\pi(\bomega_{v}^\top\x-\bomega_{u}^\top\x')}\Psi([\bomega_{u-1}, \bomega_u], [\bomega_{v-1}, \bomega_v]) = \sum_{1\leq u, v\leq Q} \alpha_{uv} e^{2i\pi(\bomega_u^\top\x-\bomega_v^\top\x)}.
\end{align}
Given the positive definiteness of $\Psi(\cdot,\cdot)$, the matrix $(\alpha_{uv})_{u,v=1}^Q$ is positive semidefinite. the Darboux sum takes a ``generalized sparse spectrum'' form: $k_{\text{GSS}}(\x, \x') = \phi(\x)^\dag \B \phi(\y)$. It is an uniform approximator of the double integral on a compact set $[\bomega_0, \bomega_Q]\times[\bomega_0, \bomega_Q]$, which converges to $k_0$ as $[\bomega_0, \bomega_Q]\times[\bomega_0, \bomega_Q]$ covers the entire frequency domain.\par
Given the expressiveness of the generalized sparse spectrum kernel, we can similarly smooth the spectral representation by multiplying with $k_{\text{LS}}(\x\circ\gamma, \x'\circ\gamma)$, and add more flexibility by translating the input, which gives the final harmonizable mixture kernel form.
\end{proof}
It is worth noting that the theorem can be strengthened from positive semidefinite Hermitian matrices $\B_p$, to non-negative valued positive semidefinite matrices. This is an immediate result from the ``phase shift'' of the Fourier transform.
\section*{4 \: \: Expressiveness of product spectral kernels}
The spectral mixture product (SMP) kernel \citep{wilson2014fast} is a variant of the spectral mixture kernel, where the inner product inside the cosine function is decomposed into a product of cosines, which makes each spectral component a product kernel.
\begin{align}
    k_{\text{SMP}}(\btau) &= \sum_{q=1^Q} w_q^2\prod_{d=1}^De^{-2\pi^2\sigma_d^2\tau_d^2}\cos(2\pi\mu_{qd}\tau_d).
\end{align}
Spectral mixture product kernel is used in multidimensional pattern discovery for its added scalability \citep{wilson2014fast}. However, it is not as expressive as the original spectral mixture kernel. We see the product of cosines can be decomposed as follows
\begin{align}
    \prod_{d=1}^D\cos(2\pi\mu_{qd}\tau_d) &= \dfrac{1}{2^D} \sum_{\mathbf{b}\in\{-1, 1\}^D} e^{2i\pi(\mathbf{b}\circ\bmu)^\top\btau}.
\end{align}
Therefore, product spectral kernels are spectral mixture kernel with additional symmetry constraint: $\psi_k(\omega) = \psi_k(b\circ\omega), \forall b\in\{-1, 1\}^D$. Note that this constraint is stricter than the constraint for an arbitrary stationary kernel $\psi_k(\omega) = \psi_k(-\omega)$. We conclude that spectral mixture product kernel shall behave as well as spectral mixture kernel when we underlying covariance has a spectral distribution that is symmetrical with respect to every ``axis''.\par
For multidimensional harmonizable spectral kernel, we can utilize enhanced scalability when we similarly replace the cosine term with a product of cosines with respect to every dimension, which leads to similar stronger symmetry of the generalized spectral distribution $\Psi(\omega, \xi) = \Psi(b_1\circ\omega, b_2\circ\xi), \forall b_1, b_2 \in\{-1, 1\}^D$.\par
When we use product spectral kernel in replacement of original spectral kernels, there is a tradeoff between scalability and expressiveness: product spectral kernels offer additional scalability for the cost of reduced expressiveness based on symmetry of the (generalized) spectral distribution.
\section*{5 \: \: Interpreting generalized spectral mixture kernel}
The \emph{generalized spectral mixture kernel} (GSM) \citep{remes2017non} is a nonstationary generalization of the stationary spectral mixture kernel. The functional formulation makes the kernel able to handle complex structure in the input. It is formulated as 
\begin{align}
    k_{\text{GSM}}(x,x') &= \sum_{q=1}^Qw_q(x)w_q(x') k_{\text{Gibbs, q}}(x, x')\cos(2\pi(\mu_q(x)x-\mu_q(x')x')),\\
    k_{\text{Gibbs, q}}(x, x') &= \sqrt{\dfrac{2l_q(x)l_q(x')}{l_q(x)^2+l_q(x')^2}}\exp(-\dfrac{(x-x')^2}{l_q(x)^2+l_q(x')^2}),
\end{align}
where functions $w_q(x)$, $\mu_q(x)$, $l_q(x)$ have GP priors, encoding a spectrogram with $w_q(x)$ denoting the magnitude of the frequency, $\mu_q(x)$, and $l_q(x)$ denoting the mean and variance of the frequency components. We propose that this kernel first projects input using some unknown feature map, and then assume stationary in the projected space and fit a stationary spectral mixture kernel. Consider the kernel $k_{FSS}(\x,\x')=\cos(g(\x)-g(\x'))$ with an arbitrary function $g: \R^D \mapsto\R$. Assuming $g(\cdot)$ lies within some RKHS $\mathcal{H}$, then $g(x) = \langle g, K(x,\cdot)\rangle_\mathcal{H}$ is an inner product between a ``constant vector'' $g$ and the projected input $K(x,\cdot)$, therefore the kernel $k_{FSS}$ generalizes sparse spectrum kernel by projecting the data with a feature map first. The GSM kernel then multiplies $k_{FSS}$ with a Gibbs kernel, implying an unknown mixture model on the spectrum induced by the projected space. \par
\begin{figure}[h]
    \centering
    \includegraphics[width=.6\textwidth]{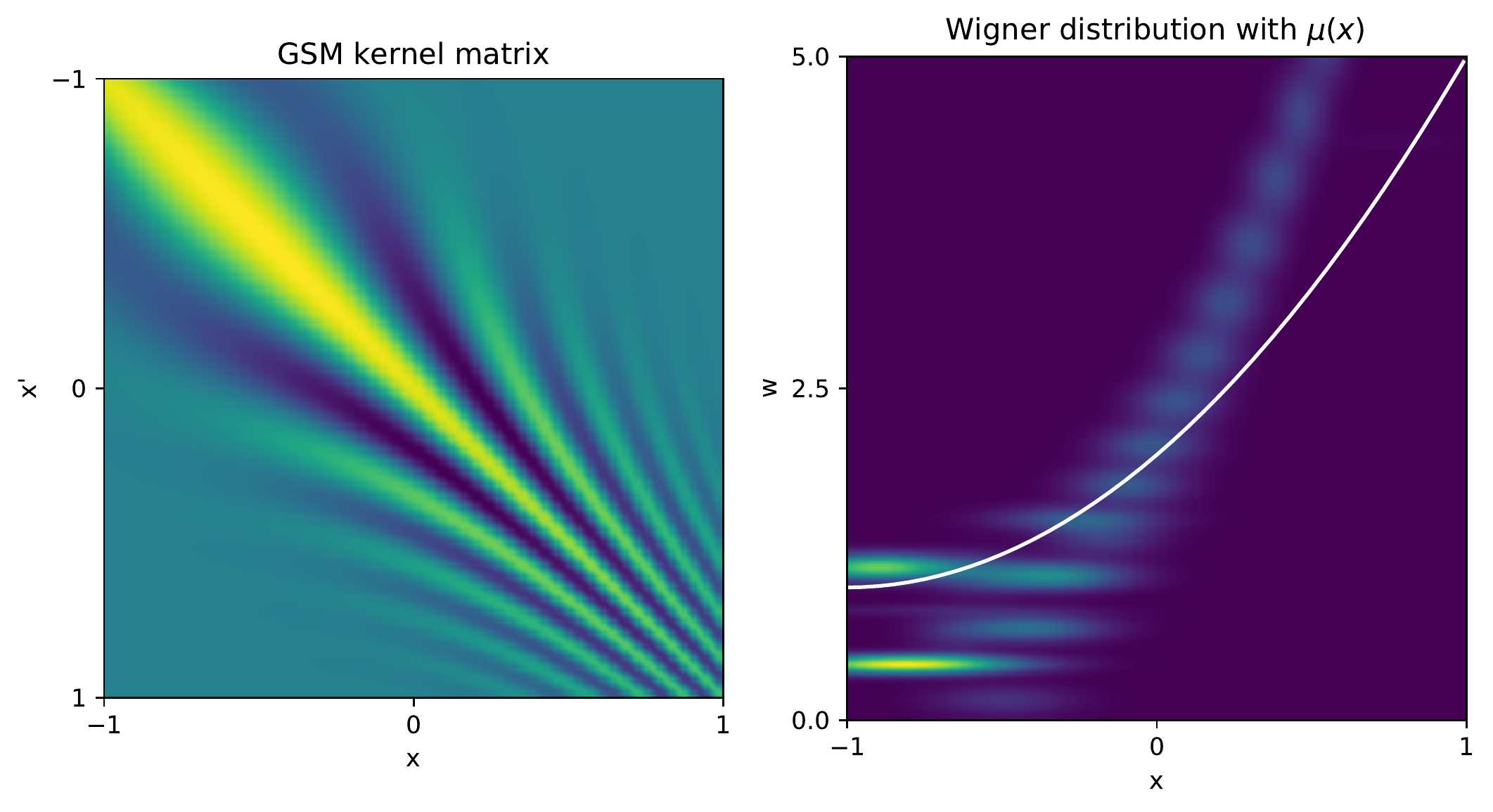}
    \caption{Wigner distribution of the approximation of a GSM kernel}
    \floatfoot{The white line denotes the $\mu(x)$ corresponding to frequency of the spectrogram.}
    \label{fig:my_label}
\end{figure}
However, the intuitive interpretation of the underlying spectrogram might be an inaccurate way to interpret GSM kernel. When we approximate a GSM kernel with HM kernel, the Wigner distribution of the HM kernel does not quite correspond to the spectrogram interpretation: the mean of the frequency components are ``stretched'', when $x$ approaches $1$, the actual local frequency is higher than what the function $\mu(x)$ suggests. GSM kernel seems to keep a biased account of the frequency information.

While the harmonizable mixture kernel handles nonstationarity in the input directly, the GSM kernel is equally valid -- it projects the input space to a feature space, and then assumes stationarity on the feature space.
\section*{6 \: \: Experiment details}
The models are implemented in Python using the GPFlow framework \citep{GPflow2017}. We implemented the \emph{harmonizable mixture kernel}, two sparse GP models with variational Fourier features (namely the variational lower bound for sparse GP regression \citep{titsias2009variational} and the stochastic variational Gaussian process \citep{hensman2017variational}), and a natural gradient optimizer accepting complex-valued variational parameters.\par
\subsection*{6.1 \: \: Kernel recovery}
For kernel recovery, we perform stochastic gradient descent using Adam \citep{kingma2014adam}, using mean square error of random batches of data as objective function.\par
\subsection*{6.2 \: \: GP classification}
For GP classification using banana dataset, we selected a subset of data containing 500 data points, and trained a variational GP model. The full variational model is then approximated using sparse GP with inducing points and inducing frequencies. \par
The inducing points are initialized using K-means clustering, and the inducing frequencies are initialized using the frequency means suggested in the trained HMK, with an added Gaussian noise. We ran each model with 5 random initializations and pick the model with highest classification accuracy on the training set.\par
For the training of sparse GP model, we first trained the variational parameters with natural gradients for 200 iterations. We then jointly train the inducing variables and variational parameters with 700 alternating rounds of optimization using respective natural gradient optimizers and Adam (such approach is suggested in \citep{salimbeni2018natural}). 
\subsection*{6.3 \: \: GP regression}
For GP regression with solar irrandiance, we used the same partition of training and test set in experiments in \citep{gal2015improving} and \citep{hensman2017variational}. We further standardize the X-axis for numerical stability of the variational Fourier features. We used sparse GP regression \citep{titsias2009variational}, where the model is modified to allow for VFF with the harmonizable mixture kernel.\par
For GP regression with Gaussian kernel, we used 50 inducing points initialized with K-Means, and initialized the kernel hyperparameters using 5 increasing lengthscales. The model is chosen using log-likelihoods on the training set.\par
With an assumption of smoothness of the underlying data, we used the residual value of the training data minus the predicted value of the previous model, and used a discrete Fourier transform on 6 subdivisions of data. The SM kernel has 3 frequency components initialized with respectively the highest two frequency in the discrete Fourier transform and the 0 frequency. This is initialization is then added with Gaussian noise and optimized.\par
The HMK for GP regression has a total of $P=6$ components, with $Q=3$ frequency values for each components. The input shifts $\x_p$ are initialized using K-means clustering, and the frequency values are 0, the highest density frequency obtained in discrete Fourier transform, and random values. We ran the sparse GP model with inducing points for some iterations and then ran variational Fourier features centered around the frequency values.\par
\end{document}

%% file: math_macros.tex
\usepackage{mathtools}
\newcommand{\x}{\mathbf{x}}
\newcommand{\z}{\mathbf{z}}

\renewcommand{\u}{\mathbf{u}}
\newcommand{\y}{\mathbf{y}}

\renewcommand{\t}{\mathbf{t}}

\newcommand{\f}{\mathbf{f}}

\newcommand{\btau}{\boldsymbol{\tau}}
\newcommand{\A}{\mathbf{A}}

\newcommand{\K}{\mathbf{K}}

\newcommand{\B}{\mathbf{B}}

\newcommand{\R}{\mathbb{R}}
\newcommand{\0}{\mathbf{0}}

\newcommand{\N}{\mathcal{N}}
\newcommand{\GP}{\mathcal{GP}}

\newcommand{\bmu}{\boldsymbol{\mu}}

\newcommand{\bSigma}{{\boldsymbol{\Sigma}}}

\newcommand{\bomega}{\boldsymbol{\omega}}
\newcommand{\bxi}{\boldsymbol{\xi}}
\newcommand{\tx}{\widetilde{\mathbf{x}}}
\newcommand{\bgamma}{\boldsymbol{\gamma}}

\DeclareMathOperator{\cov}{\textbf{cov}}

\newtheorem{theorem}{Theorem}

\newcommand{\norm}[1]{\left\lVert#1\right\rVert}

\makeatletter \renewcommand\d[1]{\ensuremath{%
  \;\mathrm{d}#1\@ifnextchar\d{\!}{}}}
\makeatother
\newenvironment{proofskch}{\noindent\textit{Proof sketch.}}{\hfill$\square$}

\theoremstyle{definition}
\newtheorem{definition}{Definition}